\documentclass[runningheads]{llncs}

\usepackage[T1]{fontenc}
\usepackage{lmodern}
\usepackage{graphicx}
\usepackage{pifont}
\usepackage{amsmath,amssymb}
\usepackage{booktabs}
\usepackage{cite}
\usepackage{multirow}
\usepackage{xcolor}
\usepackage{tikz}
\usetikzlibrary{arrows.meta}
\definecolor{chanc}{HTML}{2A78D6}
\definecolor{gatec}{HTML}{EB6834}
\usepackage[hidelinks]{hyperref}
\usepackage{enumitem}
\setlist{nosep,leftmargin=*}
\newcommand{\method}{AlignGate}

\newif\ifdraft
\draftfalse
\ifdraft
  
  \newcommand{\todo}[1]{\textcolor{red}{[TODO: #1]}}
\else
  
  \newcommand{\todo}[1]{}
\fi

\newcommand{\GF}{\textsc{GF}}
\newcommand{\NW}{\textsc{NW}}
\newcommand{\zL}{\mathbf{z}_{L}}
\newcommand{\zH}{\mathbf{z}_{H}}
\newcommand{\astar}{\alpha^{*}}
\newcommand{\anceo}{\alpha^{\textsc{nce}}}
\newcommand{\Biaspm}{\mathrm{Bias}_{\pm}}
\newcommand{\BiasSq}{\mathrm{Bias}^{2}}
\newcommand{\Var}{\mathrm{Var}}
\newcommand{\sg}{\operatorname{sg}}
\newcommand{\Lnce}{\ell_{\textsc{nce}}}
\newcommand{\LNCE}{\mathcal{L}_{\textsc{nce}}}

\begin{document}

\title{Why Node-wise Gates Fail Without Labels:\\
A Bias--Variance Analysis of Spectral Graph Contrastive Learning}
\titlerunning{Why Node-wise Gates Fail Without Labels}

\author{Zhuolong Li \and Boxue Yang \and Haopeng Chen\thanks{Corresponding author.}}
\authorrunning{Z. Li et al.}
\institute{School of Computer Science, Shanghai Jiao Tong University, Shanghai, China\\
\email{\{li.zhuolong,yangboxue,chen-hp\}@sjtu.edu.cn}}

\hypersetup{pdftitle={Why Node-wise Gates Fail Without Labels: A Bias-Variance Analysis of Spectral Graph Contrastive Learning},pdfauthor={Zhuolong Li, Boxue Yang, Haopeng Chen}}
\maketitle

\begin{abstract}
Spectral graph contrastive learning builds a low-frequency and a
high-frequency channel and fuses them globally, with one rule shared by all
nodes. A node-wise
gate is the natural refinement: it is at least as expressive, and
node-adaptive filtering helps in supervised learning. However, without
labels, the gate falls behind global fusion on homophilic graphs and is
unstable across seeds. Measuring its error against a label-derived oracle
gate, we decompose it into bias and variance and trace each to a cause with
controlled interventions. The bias comes mainly
from the uniformity pressure of the contrastive loss, which favors the more
dispersed high-frequency channel while labels on homophilic graphs favor low
frequencies. The variance comes from a positive feedback loop between the
gate and the channels it weights, which locks in seed-dependent choices
early. The two sources respond to separate interventions. These findings
lead to \method{}, which trains each channel with its own
contrastive loss and the gate with alignment only.
\method{} removes the high-frequency bias on homophilic graphs, cuts the
variance by about two thirds, and outperforms global fusion and
spectral baselines on nine benchmarks and a large graph.
\end{abstract}

\section{Introduction}
\label{sec:intro}

Graph contrastive learning (GCL) learns node representations without labels
by making the representations of two augmented views of a graph agree while
keeping different nodes apart~\cite{dgi,grace,gca}. Standard message passing
acts as a low-pass filter, which suits homophilic graphs but discards the
high-frequency signal that matters near heterophilic
boundaries~\cite{zhu2020beyond}. Spectral GCL methods therefore build a
low-frequency and a high-frequency channel with low- and high-pass graph
filters~\cite{greet,polygcl,s3gcl,hlcl,loha,dpgcl,arnoldigcl}. These methods
fuse the channels \emph{globally}: the fusion rule is shared by all nodes,
whether it is a pair of scalar weights~\cite{polygcl,loha,arnoldigcl},
concatenation~\cite{greet}, or the use of a single channel at
inference~\cite{hlcl,s3gcl,dpgcl}. We call the simplest instance, one learnable
coefficient for all nodes, \emph{global fusion} (\GF{}).

A global rule looks too coarse: nodes in homophilic regions should rely
on low frequencies, and nodes on heterophilic boundaries on high frequencies.
The natural refinement is a \emph{node-wise gate} that sets the mixing weight
for each node. A node-wise gate is at least as expressive as \GF{}, since it
contains \GF{} as a special case, and node-adaptive filtering is known to work
in supervised learning~\cite{fagcn,acmgcn,nodemoe}. One would expect the
same benefit in the self-supervised setting.

This expectation fails. With the backbone, objective, and hyperparameters
held fixed, a gate trained with the contrastive loss (\NW{}) falls below \GF{}
on homophilic graphs and is about three times as unstable across pretraining
seeds. Yet an oracle gate that takes each node's label-optimal coefficient
$\astar_v$ beats \GF{} on the same channels by 0.4--2.9 points.
Fig.~\ref{fig:intro} isolates the effect on synthetic graphs: as one group of
nodes becomes more homophilic, the gain an oracle gate could attain grows to 5
points, while the gain the learned gate realizes drops to $-2.2$ points. The
gate is not merely noisy; it moves toward high frequencies exactly where labels
call for low frequencies.

\begin{figure}[t]
  \centering
  \includegraphics[width=\textwidth]{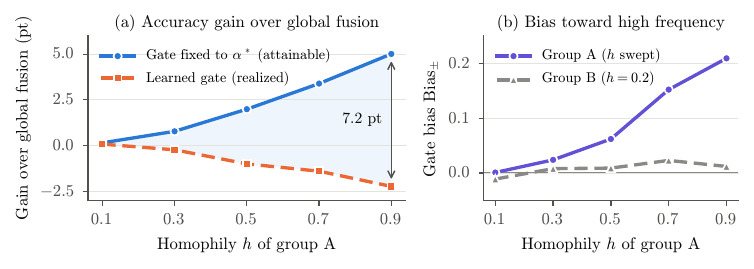}
  \caption{Mixed contextual stochastic block models~\cite{csbm}: group A has
  homophily $h$, group B is fixed at $h=0.2$. (a)~Gain over global fusion of
  the oracle gate ($\astar$) and of the learned gate. (b)~Gate bias toward
  high frequency on group A; group B is a control. Five graphs per $h$.}
  \label{fig:intro}
\end{figure}

We measure the gate's error against $\astar$ and split it exactly into a bias
and a variance term across seeds. Controlled interventions trace both to the
gate's training signal. The \emph{bias} comes mainly from the uniformity term of the
contrastive loss~\cite{wang2020understanding}, which rewards representations
that are spread apart. Low-pass filtering makes neighbors similar, so the
high-frequency channel is more uniform, and a node-wise gate can exploit this
preference node by node, whereas one global coefficient cannot.
Down-weighting the uniformity term in the gate's training signal alone, still
without labels, brings \NW{} to or above \GF{}. The
\emph{variance} comes from positive feedback: a channel with a larger gate
weight receives more gradient, improves, and is weighted even more, so
initialization noise is locked in within ten epochs. The two sources
respond to separate interventions.

Each cause suggests one change, and together they give \method{}, in which
uniformity still prevents collapse in the channels but no longer chooses
between them.
The analysis also explains why an alternative regularizer, which guides the gate
with detached per-channel losses, helps only partially. Our contributions are:
\begin{itemize}
  \item \emph{Diagnosis.} A failure mode of node-wise gates in spectral GCL:
  a gate trained with the contrastive loss underperforms global fusion on
  homophilic graphs and is unstable across seeds, although node-wise fusion of the same
  channels can win. We decompose the gate's error into bias and variance and
  support a cause for each with interventions.
  \item \emph{Method.} \method{}, which trains each channel with its own
  contrastive loss to cut the feedback loop and the gate with alignment only to
  remove uniformity from its training signal. It changes only which loss
  updates which parameters and keeps the backbone unchanged.
  \item \emph{Evaluation.} On nine benchmarks and arXiv-year, \method{}
  achieves the highest mean accuracy among all compared self-supervised
  methods, and ablations confirm
  that each component reduces the error term it targets.
\end{itemize}

\section{Related Work}
\label{sec:related}

\paragraph{Graph contrastive learning and spectral GCL.}
In graph self-supervised learning, DGI~\cite{dgi} contrasts local and global
summaries, GRACE~\cite{grace} and GCA~\cite{gca} contrast two augmented views,
and BGRL~\cite{bgrl} and CCA-SSG~\cite{ccassg} avoid negatives through
bootstrapping or decorrelation. Because message passing is
low-pass and hurts heterophilic graphs~\cite{zhu2020beyond}, several methods
learn from both ends of the spectrum. GREET~\cite{greet} separates homophilic
and heterophilic edges, and SP-GCL~\cite{spgcl} targets both regimes with a
single-pass design. Spectral GCL methods build explicit low- and
high-frequency channels: PolyGCL~\cite{polygcl} with learnable polynomial
filters, S3GCL~\cite{s3gcl} with hybrid spectral--spatial encoders,
ArnoldiGCL~\cite{arnoldigcl} with Arnoldi-based Chebyshev interpolation, and
HLCL~\cite{hlcl}, LOHA~\cite{loha}, and DPGCL~\cite{dpgcl} with paired low-
and high-pass views. All of them fuse the channels globally, with one rule
shared by all nodes. We do not propose another channel construction. Instead, we ask why the obvious refinement, a node-wise gate, fails in
this setting, and how to train it so that it does not.

\paragraph{Node-adaptive filtering with labels.}
In supervised learning, letting each node or edge choose its frequency
response is well established. FAGCN~\cite{fagcn} learns signed edge weights
that mix low- and high-frequency signals, ACM-GCN~\cite{acmgcn} mixes
low-pass, high-pass, and identity channels per node, and
Node-MoE~\cite{nodemoe} routes nodes to different filters.  In all of these, the
classification loss tells the gate which channel helps each node. Our
contribution is not the gate, which these works already use, but its failure
in GCL, where no labels are available and the contrastive objective supplies
the gate's training signal: this signal is biased toward high frequencies
and, through feedback, unstable.

\paragraph{What contrastive objectives prefer.}
Wang and Isola~\cite{wang2020understanding} decompose contrastive losses into
alignment and uniformity, and Chen et al.~\cite{chen2021intriguing} show that
contrastive learning can let easy features suppress others. Closer to our
setting, work on learning views or augmentations finds that the contrastive
objective alone is a poor guide for choosing what to contrast: InfoMin~\cite{tian2020infomin}
argues that good views must be chosen with the downstream task in mind, and
graph methods that learn augmentations, such as JOAO~\cite{joao} and
AD-GCL~\cite{adgcl}, rely on min--max formulations or explicit constraints to
keep the learned choices useful. We study a different choice, the per-node
mixing of two spectral channels. The contrastive objective misguides it
because uniformity favors the more dispersed channel rather than the more informative one. We fix this without adversarial training by assigning each
term of the loss to the parameters it should train.

\paragraph{Gating and mixture-of-experts.}
Learned gates over experts are prone to self-reinforcing imbalance, which
mixture-of-experts models counter with load-balancing losses or noisy
routing~\cite{shazeer2017moe}. The positive feedback we observe is related
but different: it is not that a few experts take over the whole population,
but that each node's choice is locked in by initialization noise, so that the
allocation differs from seed to seed. Balancing the load would not remove
this; cutting the gradient path from the gate to the channels does.

\section{Preliminaries}
\label{sec:prelim}

\paragraph{Channels and fusion.}
Let $G=(\mathbf{A},\mathbf{X})$ be a graph with $n$ nodes, normalized
Laplacian $\mathbf{L}$, and rescaled Laplacian
$\widetilde{\mathbf{L}}=2\mathbf{L}/\lambda_{\max}-\mathbf{I}$. Following
spectral GCL~\cite{polygcl}, we build a low-frequency and a high-frequency
channel with learnable Chebyshev filters and a shared projection network
$f_\theta$:
\begin{equation}
\mathbf{Z}_L = f_\theta\Big(\sum_{k=0}^{K} w_k^{L}\,T_k(\widetilde{\mathbf{L}})\mathbf{X}\Big),
\qquad
\mathbf{Z}_H = f_\theta\Big(\sum_{k=0}^{K} w_k^{H}\,T_k(\widetilde{\mathbf{L}})\mathbf{X}\Big),
\label{eq:channels}
\end{equation}
where $T_k$ is the $k$-th Chebyshev polynomial and the coefficients are
parameterized so that the response of the $L$ channel is non-increasing and
that of the $H$ channel non-decreasing in frequency~\cite{polygcl}. We write
$\mathbf{z}_{L,v}$ and $\mathbf{z}_{H,v}$ for the rows of node $v$.

The two channels are merged as
$\mathbf{z}_v(m)=m\,\mathbf{z}_{L,v}+(1-m)\,\mathbf{z}_{H,v}$, where a larger $m$ puts more
weight on low frequencies. \emph{Global fusion} (\GF{}) learns one coefficient
$m_v=\alpha$ for all nodes. A \emph{node-wise gate} predicts
\begin{equation}
m_v=\sigma\big(\mathrm{MLP}([\mathbf{z}_{L,v}\,\|\,\mathbf{z}_{H,v}])\big)
\label{eq:gate}
\end{equation}
from the node's own channel representations. In the naive version (\NW{}),
the channels, the projection, and the gate are trained jointly with the
contrastive loss on the fused representations.

\paragraph{Contrastive objective.}
Let $\mathbf{z}_v^{+}$ be the fused representation of $v$ in a view
augmented by edge dropping and feature masking, $s(\cdot,\cdot)$ the cosine similarity,
and $\tau$ a temperature. The node-level InfoNCE
loss~\cite{grace} splits into two terms:
\begin{equation}
\Lnce(\mathbf{z}_v,\mathbf{z}_v^{+})=\underbrace{-\,s(\mathbf{z}_v,\mathbf{z}_v^{+})/\tau}_{\text{alignment}}
\;+\;\underbrace{\log\sum_{u}\exp\big(s(\mathbf{z}_v,\mathbf{z}_u^{+})/\tau\big)}_{\text{uniformity}},
\label{eq:infonce}
\end{equation}
where $u$ ranges over the nodes of the augmented view (including $v$). Wang
and Isola~\cite{wang2020understanding} show that, as the number of negatives
grows, this loss approaches the sum of an alignment term and a term that is
minimized by uniformly distributed representations. Following common usage, we call the two terms alignment and uniformity. The alignment term asks the representation to be stable under augmentation, while the uniformity term pushes the representations of different nodes apart, which
prevents collapse. \NW{} minimizes
$\LNCE=\frac{1}{n}\sum_v\Lnce(\mathbf{z}_v,\mathbf{z}_v^{+})$ over all parameters, including the gate.

As \GF{} is a constant gate, node-wise fusion can only
help:

\begin{proposition}
\label{prop:regret}
Let $R_v(m)$ be any per-node risk of fusing node $v$ with coefficient $m$.
Then $\min_{\alpha}\frac{1}{n}\sum_v R_v(\alpha)\ge\frac{1}{n}\sum_v\min_{m_v}R_v(m_v)$.
If, in addition, each $R_v$ grows at least quadratically with curvature $\mu$
away from its minimizers, and fractions $p_-,p_+$ of the nodes have all their
minimizers in $[0,a_-]$ and $[a_+,1]$ with $\Delta=a_+-a_->0$, then the gap is
at least $\frac{\mu}{2}\frac{p_-p_+}{p_-+p_+}\Delta^{2}$.
\end{proposition}

\noindent The first claim holds because $\min$ of a sum is at least the sum of
the minima. For the second, quadratic growth lower-bounds the gap by
$\frac{\mu}{2}\min_\alpha\big[p_-(\alpha-a_-)_+^2+p_+(a_+-\alpha)_+^2\big]$,
which is minimized at $\alpha=(p_-a_-+p_+a_+)/(p_-+p_+)$.
Proposition~\ref{prop:regret} says what a node-wise gate \emph{could}
achieve. It says nothing about what a gate \emph{learns} when its only
training signal is the contrastive loss. The next section shows that the two
differ sharply.

\section{Why Node-wise Gates Fail}
\label{sec:analysis}

\subsection{Setup and Measures}
\label{sec:measures}

\paragraph{Protocol.}
To separate the gate's behavior from data-split noise, all analysis
experiments use one fixed data split (60/20/20) and ten pretraining seeds,
and report means and standard deviations across seeds. We study five homophilic graphs (Cora, Citeseer, Pubmed,
Amazon-Photo, Coauthor-CS~\cite{shchur2018pitfalls}) and three heterophilic
graphs (Chameleon, Squirrel, Actor). Table~\ref{tab:bias} covers the three citation graphs and the three
heterophilic graphs. Mechanism experiments focus on Cora, Pubmed,
and Squirrel. Interventions use the configuration tuned for \NW{} and change a
single factor.

\paragraph{Label-optimal coefficient.}
After pretraining, we freeze $\mathbf{Z}_L$ and $\mathbf{Z}_H$ and, for each
$\alpha\in\{0,0.1,\dots,1\}$, train a linear probe on the globally fused
representations $\mathbf{z}_v(\alpha)$ of the training nodes. For every
validation and test node we record the predicted probability of its true
class and take the maximizing $\alpha$. Ties are averaged, and nodes misclassified for every $\alpha$ are excluded from Bias and Var
(7--28\% of nodes); Oracle and Shuffled accuracies use all test nodes. The
label-optimal coefficient $\astar_v$ is the median over the ten seeds. It is
recomputed for every variant from that variant's frozen channels and used
\emph{only} for analysis, never for training or model selection. The target is
stable across seeds: its
across-seed Spearman correlation is at least 0.55 on every dataset, and it
correlates positively with local homophily $h_v$ (0.27--0.60), as intuition
suggests. The \emph{contrastive-optimal} coefficient $\anceo_v$ is defined in
the same way, with the node's InfoNCE loss in place of the class
probability.

\paragraph{Bias--variance decomposition.}
Let $m_v^{(s)}$ be the gate of node $v$ under seed $s$ and $\bar m_v$ its
mean over seeds. Averaged over nodes, the squared error of the gate splits
exactly into
\begin{equation}
\frac{1}{|V|}\sum_{v}\mathbb{E}_s\big[(m_v^{(s)}-\astar_v)^2\big]
=\underbrace{\frac{1}{|V|}\sum_{v}(\bar m_v-\astar_v)^2}_{\BiasSq}
+\underbrace{\frac{1}{|V|}\sum_{v}\mathrm{Var}_s\big(m_v^{(s)}\big)}_{\Var}.
\label{eq:decomp}
\end{equation}
We also report the signed bias $\Biaspm=\frac{1}{|V|}\sum_v(\astar_v-\bar m_v)$, where $\Biaspm>0$ means the gate puts more weight on the high-frequency channel than
$\astar$ does.

\subsection{Symptoms and a Label-Informed Oracle}
\label{sec:capacity}

\begin{table}[t]
\caption{\NW{} against \GF{}: accuracy (\%) and its standard deviation across ten
pretraining seeds (fixed split). ``Oracle'' fixes each gate to $\astar_v$
using validation and test labels (a label-informed reference, not a method);
``Shuffled'' permutes $\astar$ across nodes.}
\label{tab:phenomenon}
\centering
\setlength{\tabcolsep}{4pt}
\begin{tabular}{lccccccc}
\toprule
 & \multicolumn{2}{c}{Accuracy} & & \multicolumn{2}{c}{Seed std} & \multicolumn{2}{c}{Gate fixed to $\astar$} \\
\cmidrule(lr){2-3}\cmidrule(lr){5-6}\cmidrule(lr){7-8}
Dataset & \GF{} & \NW{} & \NW{}$-$\GF{} & \GF{} & \NW{} & Oracle & Shuffled \\
\midrule
Cora      & 87.62 & 87.26 & $-0.36$ & 0.80 & 2.41 & 88.98 & 86.97 \\
Citeseer  & 79.96 & 79.28 & $-0.68$ & 0.79 & 2.29 & 82.01 & 77.62 \\
Pubmed    & 86.68 & 85.19 & $-1.49$ & 0.65 & 2.19 & 88.04 & 85.75 \\
Amazon-Photo & 93.34 & 91.86 & $-1.48$ & 0.92 & 2.30 & 94.61 & 90.07 \\
Coauthor-CS  & 92.48 & 91.54 & $-0.94$ & 0.95 & 2.41 & 92.84 & 90.14 \\
\midrule
Chameleon & 71.76 & 72.13 & $+0.37$ & 0.80 & 2.55 & 73.39 & 70.70 \\
Squirrel  & 57.15 & 58.10 & $+0.95$ & 0.87 & 2.93 & 60.08 & 56.94 \\
Actor     & 41.66 & 42.05 & $+0.39$ & 0.93 & 2.11 & 43.54 & 40.05 \\
\bottomrule
\end{tabular}
\end{table}

Table~\ref{tab:phenomenon} shows the two symptoms. On all five homophilic graphs \NW{} falls below \GF{}, while it gains modestly on the heterophilic
ones, and its seed standard deviation is on average about three times that of
\GF{}.

The oracle in Table~\ref{tab:phenomenon}, which fuses the same channels with
$\astar_v$, beats \GF{} by 0.4--2.9 points, whereas shuffling $\astar$ across
nodes falls below \GF{}: the gain requires the right coefficient at the right
node, not a better distribution of coefficients. Node-wise fusion of these
channels thus has room to win. Sects.~\ref{sec:bias} and~\ref{sec:variance}
show how the contrastive training signal steers the gate elsewhere.

\subsection{The Contrastive Objective Biases the Gate toward High Frequency}
\label{sec:bias}

\emph{Premise.} Low-pass filtering makes neighbors similar, so the low-frequency channel
should be less uniform in the sense of Wang and
Isola~\cite{wang2020understanding}. It is: the uniformity loss
$\log\mathbb{E}\,e^{-2\|\mathbf{z}_i-\mathbf{z}_j\|^2}$ of $\zH$ is lower
(better) than that of $\zL$ by 0.50 on Cora, 0.55 on Pubmed, and 0.18 on
Squirrel, and the mean cosine similarity across edges drops from 0.73 for
$\zL$ to 0.35 for $\zH$ on Cora. The uniformity term in
Eq.~\eqref{eq:infonce} therefore rewards weight on $\zH$, regardless of what
the labels need. \GF{} can exploit this preference only through one scalar; a
node-wise gate can exploit it node by node.

\begin{table}[t]
\caption{Where the gate points: means over analyzed nodes; $\rho$ is the
Spearman correlation between the seed-averaged gate $\bar m_v$ and
$\anceo$ or $\astar$.}
\label{tab:bias}
\centering
\setlength{\tabcolsep}{5pt}
\begin{tabular}{lcccccc}
\toprule
Dataset & $\bar m$ & $\astar$ & $\anceo$ & $\Biaspm$ & $\rho(\bar m,\anceo)$ & $\rho(\bar m,\astar)$ \\
\midrule
Cora      & 0.519 & 0.731 & 0.497 & $+0.212$ & 0.654 & 0.138 \\
Citeseer  & 0.512 & 0.674 & 0.472 & $+0.162$ & 0.635 & 0.198 \\
Pubmed    & 0.574 & 0.786 & 0.554 & $+0.212$ & 0.666 & 0.318 \\
\midrule
Chameleon & 0.302 & 0.337 & 0.315 & $+0.035$ & 0.551 & 0.366 \\
Squirrel  & 0.332 & 0.378 & 0.311 & $+0.046$ & 0.629 & 0.232 \\
Actor     & 0.395 & 0.401 & 0.377 & $+0.006$ & 0.536 & 0.202 \\
\bottomrule
\end{tabular}
\end{table}

\begin{figure}[t]
  \centering
  \includegraphics[width=\textwidth]{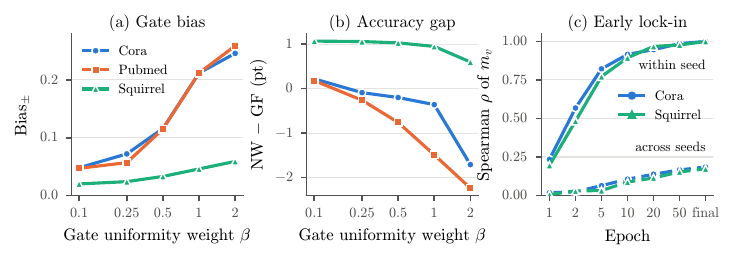}
  \caption{Mechanism evidence. (a,b)~The gate receives
  $\ell_{\text{align}}+\beta\,\ell_{\text{unif}}$ (encoder unchanged): its
  bias rises with $\beta$ and its accuracy relative to \GF{} falls. (c)~Gate
  correlation with its final value (solid) and across seeds (dashed).}
  \label{fig:mechanisms}
\end{figure}

\emph{The gate follows the contrastive optimum.} Table~\ref{tab:bias} confirms the prediction. On the homophilic graphs the
gate sits 0.16--0.21 below $\astar$, i.e., toward high frequency. On the heterophilic graphs, where labels also favor high frequency, the bias is at
most 0.05. On all six datasets the gate correlates far more with the contrastive
optimum than with the label optimum, and on the homophilic graphs it is also
far closer to it in mean.

\emph{Dose--response.} To test causation, we split the InfoNCE loss as in Eq.~\eqref{eq:infonce} and
let the gate receive only the gradient of
$\ell_{\text{align}}+\beta\,\ell_{\text{unif}}$, while the encoder is still
trained with the full loss. Fig.~\ref{fig:mechanisms}(a) shows that the bias
increases strictly monotonically with $\beta$ on all three graphs (Cora:
$+0.05$ at $\beta=0.1$ to $+0.25$ at $\beta=2$), while the variance changes
by at most 15\%. Squirrel responds weakly, as expected when labels and the
objective agree. Fig.~\ref{fig:mechanisms}(b)
shows that accuracy follows the bias: at $\beta=0.1$ \NW{} already beats \GF{} on
Cora and Pubmed, while at $\beta=2$ it falls 1.7 and 2.2 points below it.
The synthetic experiment of Fig.~\ref{fig:intro} gives the same picture from
the data side: on graphs with two node groups, the bias grows with the
homophily of one group and stays near zero for the other, and the realized
gain of the gate falls as the attainable gain rises.

\emph{Other objectives.} The bias is not specific to InfoNCE. With the original binary
cross-entropy objective of PolyGCL, which has no explicit uniformity term but
pushes representations away from those of a corrupted graph and thus exerts
a similar pressure, a naive gate shows a similar bias, $+0.19$ on Cora and $+0.18$ on Pubmed,
as Table~\ref{tab:polygcl} shows. Removing negatives altogether reduces
but does not remove it: with a BGRL-style objective~\cite{bgrl} the bias drops
to 29\% of the InfoNCE value on Cora and 60\% on Pubmed, with a
CCA-SSG-style objective~\cite{ccassg} in between. Uniformity
pressure is thus the main but not the only source of bias.

\subsection{Positive Feedback Locks the Gate In Early}
\label{sec:variance}

The variance has a different origin. In \NW{}, the gradient that reaches each
channel through the fused loss is scaled by its gate weight, so a channel that
is weighted more is trained more, becomes better for that node, and is then
weighted even more. Four observations support this loop.

\emph{Early lock-in.} Fig.~\ref{fig:mechanisms}(c) tracks the gate during
training. Within a seed, the correlation between the gate at epoch~10 and at
the end is 0.92 on Cora and 0.89 on Squirrel. Across seeds, the final gates
correlate at most 0.19. Each run commits early, and different runs commit
differently.

\emph{The feedback is causal.} We regress, for every node and channel, the
improvement of the channel's contrastive quality on its cumulative gate
weight, controlling for the initial quality, degree, and $h_v$. The
standardized coefficient is positive on all three graphs ($+0.16$ to $+0.24$,
95\% CIs excluding zero). Because this association could also run the other
way, we repeat it on Cora and Squirrel with gates assigned at random, $m_v\sim\mathcal{U}(0.2,0.8)$,
and frozen for the first 20 epochs. The coefficient shrinks but stays
positive (Cora $+0.15$, CI $[0.09,0.22]$; Squirrel $+0.09$, CI
$[0.04,0.15]$): being weighted more \emph{causes} a channel to improve.

\emph{Initialization decides where the gate lands.} Shifting only the bias of
the gate's output layer so that training starts from $m_v\approx0.2$ or
$0.8$ moves the final mean gate by 0.26, 0.25, and 0.17 on Cora, Pubmed, and
Squirrel, and the final node-wise gates correlate below 0.3 with those of the
default start. Both effects are visible at once: on Cora the run started at
$0.8$ still ends below $\astar$ (0.64 vs.\ 0.73), which is the bias, while
the node-level allocation follows the start, which is the variance.

\emph{Cutting the loop reduces the variance.} If the channels are trained by
their own contrastive losses instead of the fused loss (stop-gradient), the
variance drops by 65\% on Cora and 59\% on Squirrel and the seed standard
deviation falls to 0.82 and 0.92 points, while $\BiasSq$ changes
little. Freezing the gate for a warm-up of $k\in\{5,10,20\}$ epochs, which leaves
the channels' loss unchanged, reduces the variance monotonically in $k$.

\subsection{The Two Sources Respond to Separate Interventions}
\label{sec:additive}

\begin{table}[t]
\caption{A $2\times2$ intervention: $\beta\in\{0.1,2\}$; ``coupled'' is \NW{}
training, ``decoupled'' the stop-gradient variant. \GF{}: 87.62 (Cora),
57.15 (Squirrel).}
\label{tab:twobytwo}
\centering
\setlength{\tabcolsep}{4.5pt}
\begin{tabular}{llcccccccc}
\toprule
 & & \multicolumn{4}{c}{Cora} & \multicolumn{4}{c}{Squirrel} \\
\cmidrule(lr){3-6}\cmidrule(lr){7-10}
$\beta$ & Channels & $\BiasSq$ & $\Var$ & MSE & Acc. & $\BiasSq$ & $\Var$ & MSE & Acc. \\
\midrule
low  & coupled   & 0.0133 & 0.0292 & 0.0425 & 87.84 & 0.0183 & 0.0327 & 0.0510 & 58.22 \\
low  & decoupled & 0.0121 & 0.0097 & 0.0218 & 88.30 & 0.0205 & 0.0129 & 0.0334 & 58.71 \\
high & coupled   & 0.0704 & 0.0319 & 0.1023 & 85.91 & 0.0252 & 0.0375 & 0.0627 & 57.75 \\
high & decoupled & 0.0731 & 0.0125 & 0.0856 & 86.34 & 0.0263 & 0.0134 & 0.0397 & 58.00 \\
\bottomrule
\end{tabular}
\end{table}

Table~\ref{tab:twobytwo} crosses the two interventions. On Cora, raising
$\beta$ multiplies $\BiasSq$ by 5.3 with coupled and 6.0 with decoupled
channels while changing $\Var$ far less, and decoupling cuts $\Var$ by
61--67\% whatever $\beta$ is while leaving $\BiasSq$ almost unchanged. Each
intervention moves mainly its own term, and low $\beta$ with decoupling has the smallest error and the largest margin
over \GF{} ($+0.68$ points). On Squirrel, where labels already favor high
frequency, $\beta$ multiplies $\BiasSq$ by only 1.3--1.4, and decoupling
does most of the work. The two sources thus appear largely separate, which
motivates one component for each.

\section{\texorpdfstring{\method{}}{AlignGate}: Decoupled, Alignment-only Gating}
\label{sec:method}

Section~\ref{sec:analysis} located both failure sources in the gate's training
signal: the uniformity gradient that reaches the gate causes the bias, and
the gate-scaled gradient that reaches the channels causes the variance.
\method{} keeps the backbone of Sect.~\ref{sec:prelim} and changes only which
loss updates which parameters, as Fig.~\ref{fig:gradflow} shows.

\begin{figure}[t]
  \centering
  \includegraphics[width=\textwidth]{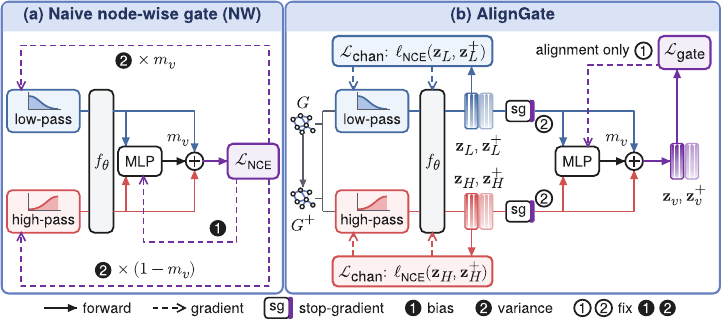}
  \caption{Gradient flow. Dashed arrows are gradients, colored by their loss;
  at sg the gradient of $\mathcal{L}_{\text{gate}}$ stops. $G^{+}$: edge
  dropping and feature masking; $\mathbf{z}_v=\mathbf{z}_v(m_v)$. (a)~In
  \NW{}, $\LNCE$ trains the gate with alignment and uniformity, pushing it
  toward the high-pass channel (\ding{182}), and trains each channel in
  proportion to its gate weight, a feedback loop (\ding{183}). (b)~In
  \method{}, each channel is trained only by $\mathcal{L}_{\text{chan}}$ and
  the gate only by alignment (\ding{172}), whose gradient stops at sg
  (\ding{173}). $m_v$ is computed on $G$ and shared by both views.}
  \label{fig:gradflow}
\end{figure}

\paragraph{Design principle.}
The two terms of Eq.~\eqref{eq:infonce} play different roles. Alignment
measures whether a representation is stable under augmentation, which is the
semantic signal of self-supervision. Uniformity is a regularizer that keeps
representations from collapsing. It constrains the representation space and
says nothing about which channel suits a node. The gate selects
between two representations rather than forming one, so it does not need
uniformity to avoid collapse, and giving it uniformity turns a label-agnostic preference for high frequencies
into a selection criterion. Uniformity should therefore stay with the
channels, where it prevents collapse, and the gate should learn from
alignment only. Likewise, the amount of training a channel receives should
not depend on how much the gate currently uses it.

\paragraph{Decoupled channel learning.}
To reduce the variance, each channel is trained only by its own contrastive
loss,
\begin{equation}
\mathcal{L}_{\text{chan}}=\frac{1}{n}\sum_{v}\Big[\Lnce\big(\mathbf{z}_{L,v},\mathbf{z}_{L,v}^{+}\big)
+\Lnce\big(\mathbf{z}_{H,v},\mathbf{z}_{H,v}^{+}\big)\Big],
\label{eq:lchan}
\end{equation}
where $\Lnce$ is the InfoNCE loss of Eq.~\eqref{eq:infonce} computed on one
channel. The fused representation sends no gradient to the filters or
$f_\theta$. Each channel's training thus does not depend on the
gate, and the feedback loop of Sect.~\ref{sec:variance} is cut.

\paragraph{Alignment-only gate learning.}
To remove the high-frequency bias, the gate is trained only to make the fused
representation agree across views:
\begin{align}
\mathcal{L}_{\text{gate}}&=\frac{1}{n}\sum_{v}\Big[1-\cos\big(\mathbf{z}_v(m_v),\,\mathbf{z}^{+}_v(m_v)\big)\Big],
\label{eq:lgate}\\
m_v&=\sigma\big(\mathrm{MLP}([\sg(\mathbf{z}_{L,v})\,\|\,\sg(\mathbf{z}_{H,v})])\big),
\nonumber
\end{align}
where $\sg$ is the stop-gradient operator and $\mathbf{z}^{+}_v(m_v)$ fuses the
augmented-view channels with the same $m_v$. Three details matter.
(i)~All channel representations entering $\mathcal{L}_{\text{gate}}$ are
detached, so $\mathcal{L}_{\text{gate}}$ updates only the gate MLP.
(ii)~$m_v$ is computed once, on the original view, and shared by both views,
so the gate cannot lower the loss by choosing differently per view.
(iii)~The cosine, rather than the Euclidean distance, prevents the gate from
lowering the loss by favoring the channel with smaller norm. On normalized
representations, $1-\cos(\mathbf{a},\mathbf{b})=\tfrac12\|\mathbf{a}-\mathbf{b}\|^2$,
so $\mathcal{L}_{\text{gate}}$ is, up to a factor of $1/2$, the alignment loss of
Wang and Isola~\cite{wang2020understanding} with exponent 2. Without negatives,
the gate has no incentive to prefer the more dispersed channel. Instead, it prefers the channel whose fused representation of $v$ is more stable under augmentation.

\paragraph{Objective and training.}
\method{} minimizes $\mathcal{L}=\mathcal{L}_{\text{chan}}+\mathcal{L}_{\text{gate}}$.
Since $\mathcal{L}_{\text{chan}}$ updates only the filters and $f_\theta$ and
$\mathcal{L}_{\text{gate}}$ only the gate MLP, a weight between them would merely rescale the gate's gradient; we use none. Training is a single joint stage: at each step we draw an
augmented view, compute both channels on both views, take a gradient step on
$\mathcal{L}_{\text{chan}}$ for the channel parameters and on
$\mathcal{L}_{\text{gate}}$ for the gate. After pretraining, $\mathbf{z}_v(m_v)$
is the node representation used downstream. Compared with \GF{}, \method{} computes the contrastive loss
once per channel and adds a node-wise MLP. Propagation still costs
$O(K|E|d)$.

\section{Experiments}
\label{sec:experiments}

We ask four questions. \textbf{RQ1}: Does \method{} improve node
classification over global fusion and existing self-supervised methods?
\textbf{RQ2}: Does each component reduce the error term it targets?
\textbf{RQ3}: Which nodes gain? \textbf{RQ4}: Do the failure and the fix carry
over to another spectral backbone?

\subsection{Setup}
\label{sec:setup}

\paragraph{Data and protocol.}
We use three homophilic citation graphs (Cora, Citeseer,
Pubmed~\cite{planetoid}), six heterophilic graphs (Cornell, Texas, Wisconsin,
Actor~\cite{geomgcn}, Chameleon, Squirrel~\cite{wikinet}), and the large
heterophilic graph arXiv-year~\cite{lim2021large} with 169{,}343 nodes, on which \GF{} and \method{} use sampled negatives
in Eq.~(\ref{eq:infonce}).
Chameleon and Squirrel contain duplicate nodes that leak between training
and test sets~\cite{platonov2023critical}. We keep them because most compared
methods report them, and none of our conclusions rests on them alone: every
mechanism experiment on Squirrel is repeated on Cora, and the variance
measures use no labels. Following
the standard linear evaluation protocol~\cite{dgi}, each model is pretrained
without labels, frozen, and evaluated by a linear classifier on the ten random
60/20/20 splits of~\cite{gprgnn,polygcl}, except on arXiv-year, where we use the
50/25/25 splits of~\cite{lim2021large}. The component and backbone studies of RQ2--RQ4
follow Sect.~\ref{sec:measures}: one fixed split and ten pretraining seeds.

\paragraph{Baselines.}
We compare with four general GCL methods, DGI~\cite{dgi},
GRACE~\cite{grace}, GCA~\cite{gca}, and CCA-SSG~\cite{ccassg}; three
heterophily-oriented methods, GREET~\cite{greet}, SP-GCL~\cite{spgcl}, and HLCL~\cite{hlcl}; and
two spectral GCL methods that use global fusion, PolyGCL~\cite{polygcl} and
S3GCL~\cite{s3gcl}. We do not include LOHA~\cite{loha} or DPGCL~\cite{dpgcl}, which have no
public implementation, or ArnoldiGCL~\cite{arnoldigcl}, whose public code we could
not run and whose reported results follow a different protocol. We also report
\GF{} and \NW{}, which share the backbone and augmentations of \method{} but
train all parameters with $\LNCE$ on the fused representation.

\subsection{RQ1: Node Classification}
\label{sec:main}

\newcommand{\std}[1]{{\scriptsize$\pm$#1}}
\begin{table}[t]
\caption{Accuracy (\%, mean $\pm$ deviation over ten splits). Best in bold, second underlined.
\GF{} and \NW{} share the backbone of \method{}.}
\label{tab:main}
\centering
\resizebox{\textwidth}{!}{%
\begin{tabular}{lccccccccc}
\toprule
 & \multicolumn{3}{c}{Homophilic} & \multicolumn{6}{c}{Heterophilic} \\
\cmidrule(lr){2-4}\cmidrule(lr){5-10}
Method & Cora & Citeseer & Pubmed & Cornell & Texas & Wisconsin & Actor & Chameleon & Squirrel \\
\midrule
DGI      & 85.88\std{0.95} & 76.44\std{0.80} & 82.13\std{0.24} & 70.82\std{7.21} & 81.48\std{2.79} & 75.00\std{2.00} & 32.09\std{1.18} & 58.23\std{0.70} & 38.80\std{0.76} \\
GRACE    & 83.27\std{0.74} & 73.79\std{0.60} & 81.71\std{0.16} & 60.66\std{11.32} & 75.74\std{2.95} & 72.13\std{2.75} & 31.97\std{1.15} & 59.52\std{1.49} & 42.68\std{0.90} \\
GCA      & 84.09\std{0.85} & 75.23\std{0.75} & 82.01\std{0.31} & 53.11\std{9.34} & 81.97\std{2.30} & 73.50\std{3.00} & 31.13\std{0.71} & 65.54\std{1.07} & 47.13\std{0.61} \\
CCA-SSG  & 87.39\std{0.89} & 79.60\std{0.71} & 84.96\std{0.20} & 78.69\std{3.44} & 87.87\std{1.64} & 82.88\std{1.50} & 34.86\std{0.56} & 60.00\std{1.20} & 41.50\std{0.72} \\
GREET    & 85.16\std{0.77} & 79.06\std{0.44} & 85.64\std{0.24} & 78.36\std{3.77} & 78.03\std{3.94} & 84.63\std{3.88} & 38.26\std{0.87} & 60.57\std{1.03} & 42.80\std{1.01} \\
SP-GCL   & 82.99\std{1.18} & 75.54\std{1.06} & 85.74\std{0.21} & 69.41\std{1.49} & 69.76\std{1.23} & 69.34\std{0.77} & 35.92\std{0.67} & 69.23\std{1.23} & 53.05\std{1.05} \\
HLCL     & 85.53\std{1.03} & 76.79\std{0.60} & 85.13\std{0.18} & 64.00\std{8.98} & 78.38\std{5.08} & 79.50\std{4.50} & 40.56\std{0.70} & 63.86\std{1.34} & 44.49\std{0.68} \\
PolyGCL  & \underline{87.57}\std{0.62} & 79.81\std{0.85} & \underline{87.15}\std{0.27} & 82.62\std{3.11} & 88.03\std{1.80} & 85.50\std{1.88} & 41.15\std{0.88} & 71.62\std{0.96} & 56.49\std{0.72} \\
S3GCL    & 87.04\std{1.25} & 77.48\std{0.80} & 86.03\std{0.37} & 81.64\std{3.06} & 85.26\std{4.66} & 85.53\std{3.38} & 40.06\std{1.58} & 71.88\std{1.91} & 56.90\std{1.37} \\
\midrule
\GF{}    & 87.49\std{0.74} & 80.21\std{0.96} & 86.93\std{0.68} & \underline{84.32}\std{2.48} & 88.11\std{2.28} & 86.80\std{2.86} & 41.61\std{0.99} & 71.23\std{1.16} & 57.05\std{0.88} \\
\NW{}    & 87.15\std{2.26} & \underline{80.56}\std{2.00} & 85.40\std{2.71} & 82.43\std{3.87} & \underline{88.38}\std{3.99} & \underline{87.60}\std{3.98} & \underline{41.76}\std{2.08} & \underline{72.09}\std{2.77} & \underline{57.92}\std{2.84} \\
\midrule
\method{} & \textbf{89.87}\std{1.05} & \textbf{82.48}\std{1.06} & \textbf{88.15}\std{0.51} & \textbf{89.46}\std{2.97} & \textbf{93.24}\std{1.91} & \textbf{88.40}\std{3.10} & \textbf{43.32}\std{0.90} & \textbf{73.22}\std{1.39} & \textbf{60.35}\std{1.38} \\
\bottomrule
\end{tabular}}
\end{table}

Table~\ref{tab:main} shows that \method{} has the highest mean accuracy on all
nine benchmarks. It improves over \GF{}, which uses the same backbone with one shared coefficient, by 1.2--5.1 points, and over the strongest external
baseline on each dataset by 1.0--6.8 points. The gains appear on homophilic and heterophilic graphs alike. On the
homophilic graphs, \NW{} falls below \GF{} on Cora and Pubmed and matches it
within noise on Citeseer, whose bias is the smallest of the three
(Table~\ref{tab:bias}). On the heterophilic graphs, \NW{} is already above
\GF{} on five of the six. With decoupled channels and an alignment-only gate,
node-wise fusion helps in both regimes. On arXiv-year,
\method{} reaches $44.14\pm0.22\%$, compared with $43.09\pm0.29\%$ for PolyGCL
and $41.80\pm0.21\%$ for \GF{}. On two further homophilic graphs,
Amazon-Photo and Coauthor-CS, where we run only the three methods that share
the backbone, \NW{} again falls below \GF{} (91.97 vs.\ 93.01 and 91.64 vs.\
92.55) and varies more than twice as much across splits. \method{} improves
over \GF{} by 1.4 points on Amazon-Photo (94.45) and matches it on
Coauthor-CS (92.71 vs.\ 92.55), where the oracle of Table~\ref{tab:phenomenon}
also leaves little room (0.4 points).

\subsection{RQ2: Each Component Reduces Its Error Term}
\label{sec:ablation}

\begin{table}[t]
\caption{Component ablation (fixed split, ten seeds; ``Std'' is across seeds).
$^\dagger$: configuration selected for that variant, which favors it. Bias and Var use each variant's own $\astar$.}
\label{tab:ablation}
\centering
\setlength{\tabcolsep}{3.2pt}
\resizebox{\textwidth}{!}{%
\begin{tabular}{lccccccccccccccc}
\toprule
 & \multicolumn{5}{c}{Cora} & \multicolumn{5}{c}{Pubmed} & \multicolumn{5}{c}{Squirrel} \\
\cmidrule(lr){2-6}\cmidrule(lr){7-11}\cmidrule(lr){12-16}
Variant & Acc. & Std & $\Biaspm$ & $\BiasSq$ & $\Var$ & Acc. & Std & $\Biaspm$ & $\BiasSq$ & $\Var$ & Acc. & Std & $\Biaspm$ & $\BiasSq$ & $\Var$ \\
\midrule
\method{}                         & \textbf{89.34} & 0.83 & $+0.019$ & 0.0073 & 0.0091 & \textbf{88.78} & 0.62 & $+0.024$ & 0.0081 & 0.0103 & \textbf{60.29} & 0.53 & $-0.046$ & 0.0099 & 0.0121 \\
\quad w/o decoupled channels      & 88.75 & 2.58 & $+0.031$ & 0.0083 & 0.0315 & 87.28 & 2.11 & $+0.039$ & 0.0107 & 0.0337 & 57.17 & 2.10 & $-0.025$ & 0.0139 & 0.0348 \\
\quad w/o alignment-only gate     & 87.58 & 0.58 & $+0.198$ & 0.0431 & 0.0109 & 86.43 & 0.43 & $+0.202$ & 0.0467 & 0.0118 & 59.19 & 0.91 & $+0.041$ & 0.0182 & 0.0145 \\
\quad w/o both (\NW{})$^\dagger$  & 87.26 & 2.41 & $+0.212$ & 0.0493 & 0.0302 & 85.19 & 2.19 & $+0.212$ & 0.0518 & 0.0317 & 58.10 & 2.93 & $+0.046$ & 0.0210 & 0.0355 \\
\quad global gate$^\dagger$       & 87.09 & 0.42 & -- & -- & -- & 86.31 & 0.38 & -- & -- & -- & 57.50 & 0.67 & -- & -- & -- \\
Heuristic regularizer$^\dagger$   & 88.58 & 0.44 & $+0.085$ & 0.0176 & 0.0066 & 87.65 & 0.70 & $+0.089$ & 0.0216 & 0.0072 & 59.20 & 0.56 & $+0.030$ & 0.0154 & 0.0071 \\
\midrule
\GF{}                             & 87.62 & 0.80 & -- & -- & -- & 86.68 & 0.65 & -- & -- & -- & 57.15 & 0.87 & -- & -- & -- \\
\bottomrule
\end{tabular}}
\end{table}

Table~\ref{tab:ablation} removes one component at a time and reports the two
error terms alongside accuracy. The pattern matches the analysis term by
term. Without decoupled channels, the variance and the seed standard deviation
return to the level of \NW{} while the bias stays small. This costs the most on Squirrel. Without the alignment-only gate, the bias returns to the level of
\NW{} while the variance stays low, and accuracy falls back to that of \GF{}
on the homophilic graphs, where the bias is largest. With both components,
both terms are small everywhere: $|\Biaspm|\le0.05$, and $\BiasSq$ is 53--85\% below that of \NW{}, so the gate is closer to $\astar$ node by node, not only on average.
Replacing the node-wise gate by one global coefficient on the same decoupled channels, trained with the alignment-only loss and tuned in the same way as \method{}, falls 2.2--2.8 points below \method{} and
stays within 0.6 points of \GF{}: decoupling does not help by itself, and the
gain requires the node-wise gate.

\paragraph{A heuristic alternative.}
Another way to train the gate without labels is to regularize it toward
channel-level evidence. The two per-channel InfoNCE losses of each node,
detached and normalized within the batch, are turned into a soft target
$\tilde m_v$ by a softmax with temperature $\tau_g$, and a binary
cross-entropy term with weight $\lambda$ pulls $m_v$ toward $\tilde m_v$. The channels are still trained by the fused loss, as in \NW{}. Since this design
does not follow from our analysis, it tests the analysis' predictions: the
detached target anchors the gate and should suppress the variance, but the
target is built from the full InfoNCE loss, uniformity included, and should
still favor the more dispersed channel. Table~\ref{tab:ablation} confirms
both. Its variance is the lowest of all variants, yet on the homophilic graphs
its bias remains about 40\% of that of \NW{}, about four times that of
\method{}, and it trails \method{} by 0.8--1.1 points. Anchoring the gate
removes the variance; only removing uniformity from the gate's training
signal removes the bias.

\subsection{RQ3 and RQ4: Further Analysis}
\label{sec:buckets}

\paragraph{RQ3: gains by node homophily.}
We split the 542 test nodes of Cora into five bins of local homophily $h_v$ of width 0.2. On the two
most homophilic groups ($h_v\ge0.6$, 411 nodes), \NW{} is 1.7 and 1.1 points
below \GF{}: its high-frequency bias hurts exactly the nodes that need low
frequencies. \method{} turns these into gains of 2.0 and 0.7 points. On the
two least homophilic groups ($h_v<0.4$, 54 nodes), \NW{} already beats \GF{}
by 4.1 and 2.0 points, and \method{} extends the margin to 6.5 and 4.3. The
only exception is the middle group ($0.4\le h_v<0.6$, 77 nodes), where
\method{} is 0.4 points below \NW{} and 2.3 points above \GF{}. \method{}
thus repairs the homophilic nodes without giving up the advantage of
node-wise fusion on heterophilic ones.

\paragraph{RQ4: another spectral backbone.}
\label{sec:backbone}
\begin{table}[t]
\caption{PolyGCL backbone with its original objective (fixed split, ten
seeds).}
\label{tab:polygcl}
\centering
\setlength{\tabcolsep}{3.3pt}
\begin{tabular}{lccccccccc}
\toprule
 & \multicolumn{3}{c}{Cora} & \multicolumn{3}{c}{Pubmed} & \multicolumn{3}{c}{Squirrel} \\
\cmidrule(lr){2-4}\cmidrule(lr){5-7}\cmidrule(lr){8-10}
Fusion & Acc. & Std & $\Biaspm$ & Acc. & Std & $\Biaspm$ & Acc. & Std & $\Biaspm$ \\
\midrule
Global (original)  & 87.90 & 0.85 & -- & 87.02 & 0.48 & -- & 56.87 & 0.78 & -- \\
Naive gate         & 86.40 & 3.00 & $+0.194$ & 85.65 & 2.01 & $+0.179$ & \textbf{57.39} & 1.69 & $+0.056$ \\
\method{} gate     & \textbf{88.32} & 0.54 & $+0.025$ & \textbf{87.78} & 0.94 & $+0.020$ & 56.52 & 0.96 & $-0.019$ \\
\bottomrule
\end{tabular}
\end{table}

To check that the failure is not specific to our backbone, we insert a
node-wise gate into PolyGCL~\cite{polygcl}, keeping its official
implementation, hyperparameters, and binary cross-entropy objective. For the \method{} gate, we add a per-channel copy of that objective and train the gate
with alignment only. Table~\ref{tab:polygcl} shows that on Cora and Pubmed the
naive gate reproduces both symptoms: it is less accurate than global fusion,
varies by 2--3 points across seeds, and has a bias of $+0.18$ to $+0.19$, even though
this objective has no explicit uniformity term. Its corrupted negatives play a
similar role. The \method{} gate removes the
bias, reduces the instability, and improves on global fusion. On Squirrel,
however, it falls below both the original model and the naive gate,
so on this backbone we only claim the fix for homophilic graphs.

\paragraph{Efficiency.}
On arXiv-year, an epoch of \method{} takes 0.48\,s and 11.1\,GiB of peak GPU
memory, against 0.34\,s and 9.7\,GiB for \GF{} in the same implementation;
the extra cost comes from the second contrastive loss. The official PolyGCL
implementation takes 0.50\,s and 50.8\,GiB.

\section{Conclusion}
\label{sec:conclusion}

A node-wise gate is at least as expressive as global fusion, yet in spectral
graph contrastive learning it underperforms on homophilic graphs and is
unstable across seeds. We traced the failure to the gate's training signal
and split the error into a bias, caused mainly by the
uniformity pressure of the contrastive objective, and a variance, caused by
positive feedback between the gate and the channels it weights. 
\method{} targets both and outperforms global fusion and spectral baselines on nine benchmarks and
arXiv-year.

\paragraph{Limitations.}
Our evidence covers two-channel spectral GCL on transductive node
classification over the standard benchmarks, not the newer heterophilic ones
of~\cite{platonov2023critical}. Removing negatives reduces but does not remove the bias, so
uniformity is its main but not its only source. On a second backbone
the fix holds on homophilic graphs but not on Squirrel. On Squirrel, the
alignment-only gate also slightly favors the smoother channel: \method{} has
$\Biaspm=-0.046$ against $+0.046$ for \NW{}, although its $\BiasSq$ is about
halved (0.0099 vs.\ 0.0210).

Our design rule, that a term preventing collapse should train the
representations but not the parameters that select among them, may also apply
to multi-view fusion and self-supervised mixtures of experts. We leave this to future work.

\bibliographystyle{splncs04}
\bibliography{refs}

\begin{thebibliography}{10}
\providecommand{\url}[1]{\texttt{#1}}
\providecommand{\urlprefix}{URL }
\providecommand{\doi}[1]{https://doi.org/#1}

\bibitem{fagcn}
Bo, D., Wang, X., Shi, C., Shen, H.: Beyond low-frequency information in graph
  convolutional networks. In: Proceedings of the AAAI Conference on Artificial
  Intelligence. vol.~35, pp. 3950--3957 (2021)

\bibitem{polygcl}
Chen, J., Lei, R., Wei, Z.: {PolyGCL}: Graph contrastive learning via learnable
  spectral polynomial filters. In: International Conference on Learning
  Representations (ICLR) (2024)

\bibitem{chen2021intriguing}
Chen, T., Luo, C., Li, L.: Intriguing properties of contrastive losses. In:
  Advances in Neural Information Processing Systems (NeurIPS). vol.~34 (2021)

\bibitem{gprgnn}
Chien, E., Peng, J., Li, P., Milenkovic, O.: Adaptive universal generalized
  pagerank graph neural network. In: International Conference on Learning
  Representations (ICLR) (2021)

\bibitem{arnoldigcl}
Co{\c{s}}kun, M., Baggag, A., Koyut{\"u}rk, M.: {ArnoldiGCL}: Graph contrastive
  learning via learnable {Arnoldi}-based guided spectral {Chebyshev} polynomial
  filters. In: Proceedings of the 31st ACM SIGKDD Conference on Knowledge
  Discovery and Data Mining (KDD). pp. 380--391 (2025)

\bibitem{csbm}
Deshpande, Y., Sen, S., Montanari, A., Mossel, E.: Contextual stochastic block
  models. In: Advances in Neural Information Processing Systems (NeurIPS).
  vol.~31 (2018)

\bibitem{nodemoe}
Han, H., Li, J., Huang, W., Tang, X., Lu, H., Luo, C., Liu, H., Tang, J.:
  Node-wise filtering in graph neural networks: A mixture of experts approach.
  arXiv preprint arXiv:2406.03464  (2024)

\bibitem{dpgcl}
Huang, R., Li, P., Zhang, K.: {DPGCL}: Dual pass filtering based graph
  contrastive learning. Neural Networks  \textbf{179},  106517 (2024)

\bibitem{lim2021large}
Lim, D., Hohne, F., Li, X., Huang, S.L., Gupta, V., Bhalerao, O., Lim, S.N.:
  Large scale learning on non-homophilous graphs: New benchmarks and strong
  simple methods. In: Advances in Neural Information Processing Systems
  (NeurIPS). vol.~34, pp. 20887--20902 (2021)

\bibitem{greet}
Liu, Y., Zheng, Y., Zhang, D., Lee, V.C.S., Pan, S.: Beyond smoothing:
  Unsupervised graph representation learning with edge heterophily
  discriminating. In: Proceedings of the AAAI Conference on Artificial
  Intelligence. vol.~37, pp. 4516--4524 (2023)

\bibitem{acmgcn}
Luan, S., Hua, C., Lu, Q., Zhu, J., Zhao, M., Zhang, S., Chang, X.W., Precup,
  D.: Revisiting heterophily for graph neural networks. In: Advances in Neural
  Information Processing Systems (NeurIPS). vol.~35, pp. 1362--1375 (2022)

\bibitem{geomgcn}
Pei, H., Wei, B., Chang, K.C.C., Lei, Y., Yang, B.: Geom-{GCN}: Geometric graph
  convolutional networks. In: International Conference on Learning
  Representations (ICLR) (2020)

\bibitem{platonov2023critical}
Platonov, O., Kuznedelev, D., Diskin, M., Babenko, A., Prokhorenkova, L.: A
  critical look at the evaluation of {GNNs} under heterophily: Are we really
  making progress? In: International Conference on Learning Representations
  (ICLR) (2023)

\bibitem{wikinet}
Rozemberczki, B., Allen, C., Sarkar, R.: Multi-scale attributed node embedding.
  Journal of Complex Networks  \textbf{9}(2),  cnab014 (2021)

\bibitem{planetoid}
Sen, P., Namata, G., Bilgic, M., Getoor, L., Galligher, B., Eliassi-Rad, T.:
  Collective classification in network data. AI Magazine  \textbf{29}(3),
  93--106 (2008)

\bibitem{shazeer2017moe}
Shazeer, N., Mirhoseini, A., Maziarz, K., Davis, A., Le, Q., Hinton, G., Dean,
  J.: Outrageously large neural networks: The sparsely-gated mixture-of-experts
  layer. In: International Conference on Learning Representations (ICLR) (2017)

\bibitem{shchur2018pitfalls}
Shchur, O., Mumme, M., Bojchevski, A., G{\"u}nnemann, S.: Pitfalls of graph
  neural network evaluation. In: Relational Representation Learning Workshop,
  NeurIPS (2018)

\bibitem{adgcl}
Suresh, S., Li, P., Hao, C., Neville, J.: Adversarial graph augmentation to
  improve graph contrastive learning. In: Advances in Neural Information
  Processing Systems (NeurIPS). vol.~34, pp. 15920--15933 (2021)

\bibitem{bgrl}
Thakoor, S., Tallec, C., Azar, M.G., Azabou, M., Dyer, E.L., Munos, R.,
  Veli{\v{c}}kovi{\'c}, P., Valko, M.: Large-scale representation learning on
  graphs via bootstrapping. In: International Conference on Learning
  Representations (ICLR) (2022)

\bibitem{tian2020infomin}
Tian, Y., Sun, C., Poole, B., Krishnan, D., Schmid, C., Isola, P.: What makes
  for good views for contrastive learning? In: Advances in Neural Information
  Processing Systems (NeurIPS). vol.~33 (2020)

\bibitem{dgi}
Veli{\v{c}}kovi{\'c}, P., Fedus, W., Hamilton, W.L., Li{\`o}, P., Bengio, Y.,
  Hjelm, R.D.: Deep graph infomax. In: International Conference on Learning
  Representations (ICLR) (2019)

\bibitem{s3gcl}
Wan, G., Tian, Y., Huang, W., Chawla, N.V., Ye, M.: {S3GCL}: Spectral, swift,
  spatial graph contrastive learning. In: International Conference on Machine
  Learning (ICML). PMLR, vol.~235, pp. 49973--49990 (2024)

\bibitem{spgcl}
Wang, H., Zhang, J., Zhu, Q., Huang, W., Kawaguchi, K., Xiao, X.: Single-pass
  contrastive learning can work for both homophilic and heterophilic graph.
  Transactions on Machine Learning Research  (2023)

\bibitem{wang2020understanding}
Wang, T., Isola, P.: Understanding contrastive representation learning through
  alignment and uniformity on the hypersphere. In: International Conference on
  Machine Learning (ICML). pp. 9929--9939 (2020)

\bibitem{hlcl}
Yang, W., Mirzasoleiman, B.: Graph contrastive learning under heterophily via
  graph filters. In: Conference on Uncertainty in Artificial Intelligence
  (UAI). PMLR, vol.~244, pp. 3936--3955 (2024)

\bibitem{joao}
You, Y., Chen, T., Shen, Y., Wang, Z.: Graph contrastive learning automated.
  In: International Conference on Machine Learning (ICML) (2021)

\bibitem{ccassg}
Zhang, H., Wu, Q., Yan, J., Wipf, D., Yu, P.S.: From canonical correlation
  analysis to self-supervised graph neural networks. In: Advances in Neural
  Information Processing Systems (NeurIPS). vol.~34, pp. 76--89 (2021)

\bibitem{zhu2020beyond}
Zhu, J., Yan, Y., Zhao, L., Heimann, M., Akoglu, L., Koutra, D.: Beyond
  homophily in graph neural networks: Current limitations and effective
  designs. In: Advances in Neural Information Processing Systems (NeurIPS).
  vol.~33, pp. 7793--7804 (2020)

\bibitem{grace}
Zhu, Y., Xu, Y., Yu, F., Liu, Q., Wu, S., Wang, L.: Deep graph contrastive
  representation learning. arXiv preprint arXiv:2006.04131  (2020)

\bibitem{gca}
Zhu, Y., Xu, Y., Yu, F., Liu, Q., Wu, S., Wang, L.: Graph contrastive learning
  with adaptive augmentation. In: Proceedings of the Web Conference (WWW). pp.
  2069--2080 (2021)

\bibitem{loha}
Zou, Z., Jiang, Y., Shen, L., Liu, J., Liu, X.: {LOHA}: Direct graph spectral
  contrastive learning between low-pass and high-pass views. In: Proceedings of
  the AAAI Conference on Artificial Intelligence. vol.~39, pp. 13492--13500
  (2025)

\end{thebibliography}

\end{document}